\documentclass{article}
\usepackage{ijcai22}

\usepackage{times}

\usepackage{soul}
\usepackage{url}
\usepackage[hidelinks]{hyperref}
\usepackage[utf8]{inputenc}
\usepackage[small]{caption}
\usepackage{graphicx}
\usepackage{amsmath}
\usepackage{booktabs}
\title{Relational Abstractions for Generalized Reinforcement Learning \\on Symbolic Problems}

\author{
Rushang Karia\and
Siddharth Srivastava
\affiliations
School of Computing and Augmented Intelligence\\
Arizona State University
\emails
\{Rushang.Karia, siddharths\}@asu.edu
}

\usepackage{algorithm}
\usepackage{algorithmic}

\usepackage{amsfonts} 

\usepackage{amsmath}
\DeclareMathOperator*{\argmax}{argmax}

\usepackage{caption}
\usepackage{subcaption}

\usepackage[switch]{lineno}

\usepackage{tikz}
\usetikzlibrary{decorations.pathreplacing,calc}

\usepackage{amsthm}

\usepackage[english]{babel}
\newtheorem{theorem}{Theorem}[section]

\usepackage{makecell}

\usepackage{booktabs}

\usepackage{pdflscape}

\usepackage{multicol}

\newcommand{\citet}[1]{\citeauthor{#1}~[\citeyear{#1}]}

\usepackage[stable]{footmisc}

\begin{document}

\maketitle

\begin{abstract}
Reinforcement learning in problems with symbolic state spaces is a challenging problem due to the need for reasoning over long horizons. This paper presents a new approach that utilizes relational abstractions in conjunction with deep learning to learn a generalizable Q-function for such problems. The learned Q-function can be efficiently transferred to related problems that have different object names and object quantities, and thus, entirely different state spaces. We show that the learned generalized Q-function can be utilized to speed up learning on related problems without an explicit, hand-coded curriculum. Empirical evaluations on a range of problems show that our method facilitates efficient transfer of learned knowledge to much larger problem instances containing many objects, often reducing the sampling requirements on larger problems by several orders of magnitude.
\end{abstract}

\section{Introduction}
\label{sec:introduction}
Deep Reinforcement Learning (DRL) has been successfully used for sequential decision making in tasks using image-based state representations \cite{DBLP:journals/corr/MnihKSGAWR13}. However,
many problems in the real world cannot be readily expressed as such and are naturally described by using factored representations in a symbolic representation language such as PDDL or RDDL \cite{DBLP:journals/jair/LongF03,Sanner:RDDL}. For example, in a logistics problem, the objective consists of delivering packages to their destined locations using a truck to carry them. Symbolic description languages such as first-order logic (FOL) can easily capture states and objectives of this scenario using predicates such as \emph{in-truck($p$)} where $p$ is a parameter that can be used to represent any package. Symbolic representations for such problems are already available in the form of databases and converting them to image-based representations would require significant human effort. Due to their practical use, symbolic descriptions and algorithms utilizing them are of keen interest to the research community.

A key difficulty in applying RL to problems expressed in such representations is that their state spaces generally grow exponentially as the number of state variables or objects increases. However, solutions to such problems can often be described by compact, easy-to-compute ``generalized policies" that can transfer to a class of problems with differing object counts, significantly reducing the sample complexity for learning a good, instance-specific policy. 

\textbf{Running Example} We illustrate the benefits of computing generalized policies using the Sysadmin$(n)$ domain that has been used in many planning competitions. A problem in this domain consists of a set of $n$ computers connected to each other in an arbitrary configuration. At any timestep, the computers can shutdown with an unknown probability distribution that depends on the network connectivity. The agent is also awarded a positive reward that is proportional to the total number of computers that are up. Similarly, at each timestep, the agent may reboot any one of the $n$ computers bearing a small negative reward or simply do nothing. In our problem setting, action dynamics are not available as closed-form probability distributions, making RL the natural choice for solving such problems.

A state in this problem is succintly described by a factored representation with boolean state variables (propositions) that describe which computers are running and their connectivity. It is easy to to see that the state spaces grow exponentially as $n$ increases. However, this problem has a very simple policy that can provide a very high reward; reboot any computer that is not running or do nothing. Even though a general policy for such a problem is easy to express, traditional approaches to RL like Q-learning cannot transfer learned knowledge and have difficulties scaling to larger problems with more computers. Our major contribution in this paper is learning a \emph{generalized, relational Q-function} that can express such a policy and use it to efficiently transfer policies to larger instances, reducing the sample complexity for learning.

Many existing techniques that compute generalized policies do so by using human-guided or automatic feature engineering to find relevant features that facilitate efficient transfer (see Sec.\,\ref{sec:related_work} for a detailed discussion of related work). For example, \citet{DBLP:conf/aips/NgP21} use an external feature discovery module to learn first-order features for Q-function approximation. API \cite{DBLP:journals/jair/FernYG06} uses a taxonomic language with beam search to form rule-based policies. 

In this paper, we approach the problem of learning generalized policies from a Q-function approximation perspective. We utilize deep learning along with an automatically generated feature list to learn a nonlinear approximation of the Q-function. Our approach learns a generalizable, relational Q-function  that facilitates the transfer of knowledge to larger instances at the propositional level. Our empirical results show that our approach can outperform existing approaches for transfer (Sec.\,\ref{sec:empirical_evaluation}).

The rest of the paper is organized as follows. The next section presents the required background. Sec.\,\ref{sec:our_approach} describes our approach for transfer followed by a description of using our algorithm for generalized reinforcement learning (Sec.\,\ref{sec:grl}).  Sec.\,\ref{sec:empirical_evaluation} presents an extensive empirical evaluation along with a discussion of some limitations of our approach. We then provide an account of related work in the area (Sec.\,\ref{sec:related_work}) and conclude by summarizing our contributions (Sec.\,\ref{sec:conclusion}).

\section{Formal Framework}
\label{sec:formal_framework}
We establish our problem in the context of reinforcement learning for Markov Decision Processes (MDPs). Adapting \citet{DBLP:journals/jair/FernYG06}, we represent relational MDPs as follows: Let $D = \langle \mathcal{P}, \mathcal{A} \rangle$ be a problem domain where $\mathcal{P}$ is a set of predicates of arity no greater than 2, and $\mathcal{A}$ is a set of parameterized action names. An MDP for a domain $D$ is a tuple 
$M = \langle O, D, S, A, T, R, \gamma, s_0\rangle$ where $O$ is a set of objects. A fact is an instantiation of a predicate $p \in \mathcal{P}$ with the appropriate number of objects from $O$. A state $s$ is a set of true facts and the state space $S$ is a finite set consisting of all possible sets of true facts. Similarly, the action space $A$ is composed of all possible instantiations of action names $a \in \mathcal{A}$ with objects from $O$. $T$ is a transition system, implemented by a simulator, that returns a state $s'$ according to some fixed, but \emph{unknown} probability distribution $P(s'|s, a)$ when applying action $a$ in a state $s$. We assume w.l.o.g. that the simulator only returns actions that are executable in a given state and that there is always one such action (which can be easily modeled using a \emph{nop}). $R: S \times A \rightarrow \mathbb{R}$ is a reward function that is also implemented by the simulator. $\gamma$ is the discount factor, and $s_0$ is the initial state. \\
\textbf{Example} The Sysadmin domain introduced in the preceding section can be described by predicates running$(c_x)$ and link$(c_x,c_y)$. The possible actions are reboot$(c_x)$, and nop$()$. $c_x$ and $c_y$ are parameters that can be grounded with objects of a specific problem. A state of a problem $M_\emph{eg}$ drawn from Sysadmin$(2)$ with connectivity $K_2$ using computer names $c_0$ and $c_1$ where only $c_0$ is up can be described as
$s_\text{eg} = \{ \text{running}(c_0), \text{link}(c_0, c_1), \text{link}(c_1, c_0) \}$. The action space of $M_\emph{eg}$ would consist of actions nop$()$, reboot$(c_0)$, and reboot$(c_1)$, with their dynamics implemented by a simulator.

A solution to an MDP is expressed as a deterministic \emph{policy} $\pi: S \rightarrow A$, which is a mapping from states to actions.
Let $t$ be any time step, then, given a policy $\pi$, the value of taking action $a$ in a state $s$ is defined as the expected return starting from $s$, executing $a$, observing a reward $r$ and following the policy thereafter \cite{DBLP:books/lib/SuttonB98}.

\begin{equation*}
    q_\pi(s, a) = \mathbb{E}_\pi \bigg[ \sum_{i=0}^{\infty} \gamma^i r_{t+i+1} \bigg| s_t=s, a_t=a \bigg]
\end{equation*}

The optimal action-value function (or Q-function) is defined as the maximum expected return over all policies for every $s \in S$ and every $a \in A$; $q_*(s, a) = \max\limits_\pi q_\pi(s, a)$. It is easy to prove that the optimal Q-function function satisfies the \emph{Bellman} equation (expressed in action-value form):
\begin{equation*}
    q_*(s, a) = \mathbb{E}_{s'\sim T} \bigg[ r_{t+1} + \gamma \max\limits_{a' \in A} q_*(s', a') \bigg| s_t=s, a_t=a \bigg]
\end{equation*}

Reinforcement learning algorithms iteratively improve the Q-function \emph{estimate} $Q(s, a) \approx q_*(s, a)$ by converting the Bellman equation into update rules. Given an observation sequence $(s_t, a_t, r_{t+1}, s_{t+1})$, the update rule for Q-learning \cite{Watkins:1989} to estimate $Q(s_t, a_t)$ is given by:
\begin{equation*}
    Q(s_t, a_t) = Q(s_t, a_t) 
    + \alpha \delta_t
\end{equation*}
where $\delta_t = r_{t+1} + \gamma \max\limits_{a' \in A} Q(s_{t+1}, a') - Q(s_t, a_t) $ is the temporal difference, TD$(0)$, error, and $\alpha$ is the learning rate. Q-learning has been shown to converge to the optimal Q-function under certain conditions \cite{DBLP:books/lib/SuttonB98}. Q-learning is an \emph{off-policy} algorithm and generally uses an $\epsilon$-greedy exploration strategy, selecting a random action with probability $\epsilon$, and following the greedy policy  $\pi(s) = \argmax_{a}Q(s, a)$ otherwise.

Let $f$ be a feature, we then define a feature kernel $\phi_f(s)$ as a function that maps a state $s$ to a set of objects.
We utilize description logic to derive and express feature kernels building upon the recent work by \citet{DBLP:conf/aaai/BonetFG19}. This is described in Sec.\,\ref{subsec:abstraction}.

\section{Our Approach}
\label{sec:our_approach}
Our goal is to compute approximate Q-functions whose induced policies generalize to problem instances with differing object counts in a way that allows
RL approaches to find good policies with minimal learning. To do so, we use the sampled state space of a small problem to automatically generate domain-specific relational abstractions that lift problem-specific characteristics like object names and numbers (Sec.\,\ref{subsec:abstraction}). Sec.\,\ref{subsec:deep_learning_for_value_approximation} describes our method of representing these abstractions as input features to a deep neural network. Finally, Sec.\,\ref{sec:grl} expands on how our algorithm, Generalized Reinforcement Learning (GRL), utilizes the deep neural network to learn approximate Q-values of abstract states and use them for transfer learning.

\subsection{Relational Abstraction}
\label{subsec:abstraction}
One challenge in Q-function approximation is the use of a representation language from which it is possible to extract features that can provide useful information for the decision making process.  We now provide a formal description of the general classes of abstraction-based, domain-independent feature synthesis algorithms that we used in this paper.

\textbf{Description Logics (DL)} are a family of representation languages popularly used in knowledge representation \cite{DBLP:books/daglib/dlbook}. Building upon many
different threads of research in learning and generalization for planning \cite{DBLP:journals/apin/MartinG04,DBLP:journals/jair/FernYG06,DBLP:conf/aaai/BonetFG19,DBLP:conf/aaai/FrancesBG21}, we chose DLs due to the diversity of features they can represent while providing a good balance between expressiveness and tractability. 

In the relational MDP paradigm, unary predicates $\mathcal{P}_1 \in \mathcal{P} $ and binary predicates $\mathcal{P}_2 \in \mathcal{P}$ of a domain $D$ can be viewed as \emph{primitive} \emph{concepts} $C$ and \emph{roles} $R$ in DL. DL includes constructors for generating \emph{compound} concepts and roles from primitive ones to form expressive languages.
Our feature set $F_{\emph{DL}}$ consists of concepts and roles formed by using a reduced set of grammar from \citet{DBLP:conf/aaai/BonetFG19}:
\begin{linenomath}
\begin{align*}
C, C' &\rightarrow \mathcal{P}_1 \mid \neg C \mid C \sqcap C' \mid \forall R.C \mid \exists R.C \mid R = R'\\
R, R' &\rightarrow \mathcal{P}_2 \mid R^{-1}
\end{align*}
\end{linenomath}

where $\mathcal{P}_1$ and $\mathcal{P}_2$ represent the primitive concepts and roles, and $R^{-1}$ represents the inverse.
$\forall R.C = \{ x \mid \forall y \text{ } R(x,y) \land C(y)  \}$ and 
$\exists R.C = \{ x \mid \exists y \text{ } R(x,y) \land C(y)  \}$. $R = R'$ denotes $\{x \mid \forall y \text{ } R(x, y) = R'(x, y)\}$. We also use $\emph{Distance}(c_1, r, c_2)$ features  \cite{DBLP:conf/aaai/FrancesBG21} that compute the minimum number of $r$-steps between two objects satisfying concepts
$c_1$ and $c_2$ respectively.

We control the total number of features generated by only considering features up to a certain complexity $k$ (a tunable hyperparameter) that is defined as the total number of grammar
rules required to generate a feature. \\
\textbf{Example} The primitive concepts and roles of the Sysadmin domain are running$(c_x)$ and link$(c_x, c_y)$ respectively. For the running example $M_\emph{eg}$, a feature $f_\emph{up} \equiv \text{running}(c_x)$ evaluates to the set of objects satisfying it, i.e., $\phi_{f_\emph{up}}(s_\emph{eg}) = \{ c_0 \}$. This feature can be interpreted as tracking the set of computers that are running (or up). It is easy to see that DL features such as  ${f_\emph{up}}$ capture relational properties of the state and can be applied to problems with differing object names and quantities.

It is clear that when using the DL features defined above, $\phi_f(s)$ evaluates to a set of objects that satisfy $f$. Thus, both $|\phi_f(s)|$ and $o \in \phi_f(s)$ are well-defined. The only exceptions are \emph{Distance} features that directly return numerical values. We simply define $|\phi_f(s)| = \phi_f(s)$ and $\forall o \in O, o \in \phi_f(s) = 0$ for such distance-based features.

We used the D2L system \cite{DBLP:conf/aaai/FrancesBG21} for generating such DL based features. We describe our use of the D2L system in more detail in Sec.\,\ref{subsec:tasks}.

\subsection{Deep Learning for Q-value Approximation}
\label{subsec:deep_learning_for_value_approximation}
Given a set of DL features $F$, another key challenge is to identify a subset $F' \subseteq F$ of features that can learn a good approximation of the Q-function.
We use deep learning utilizing the entire feature set $F$ for Q-value estimation.

Given a feature set $F$ and a domain $D$, the input to our network is a vector of size $|F| + |\mathcal{A}| +N \times |F|$ where $|\mathcal{A}|$ is the total number of actions in the domain, and $N$ is the maximum number of parameters of any action in $\mathcal{A}$. 

\begin{figure}[ht]
    \centering
    \includegraphics[width=\columnwidth]{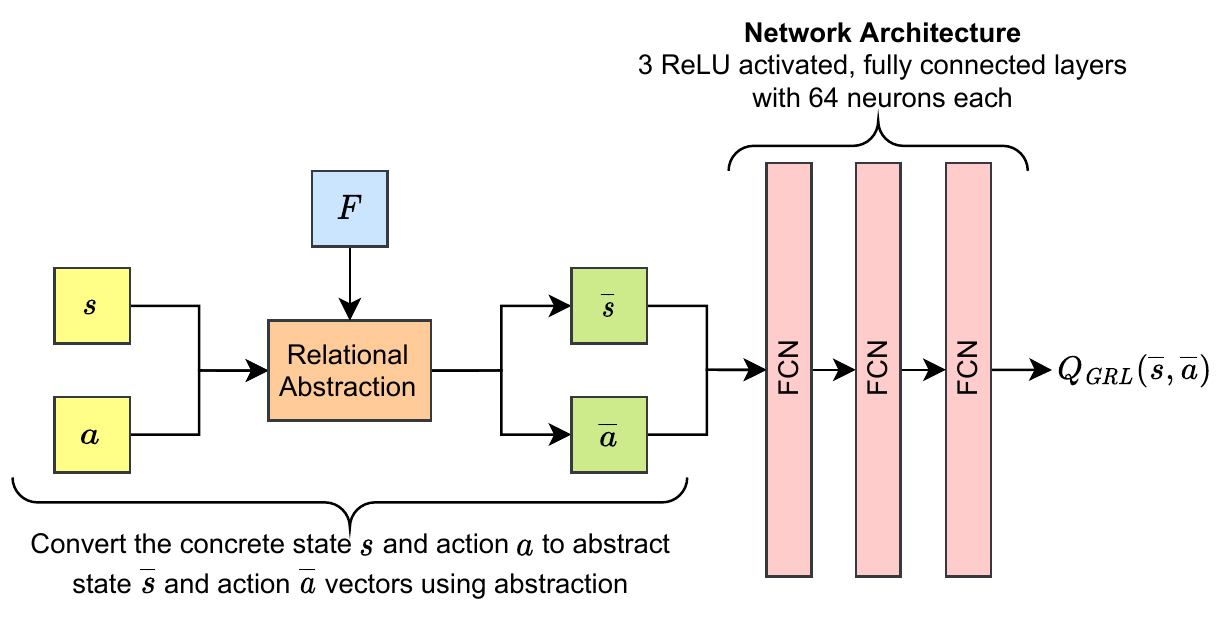}
    \caption{Our process for estimating Q-values.}
    \label{fig:network_architecture}
\end{figure}

Given a concrete state $s$, the \emph{abstract state feature vector} is defined as 
$\overline{s} = \langle |\phi_{f_1}(s)|, \ldots, |\phi_{f_n}(s)| \rangle$ for all features $f_i \in F$. Similarly, given a parameterized action $a(o_1, \ldots, o_n)$, the \emph{abstract action feature vector} is defined as a vector $\overline{a} = \langle A_\emph{name} | F_{o_1} | \ldots | F_{o_N} \rangle$ where $A_\emph{name}$ is a one-hot vector of length $|\mathcal{A}|$ encoding the action name $a$, $F_{o_i}$ is a vector of length
$|F|$ encoded with values $1_{[o_i \in \phi_{f_j}(s)]}$ for every feature $f_j \in F$, and $|$ represents vector concatenation. Together, $\overline{s}$ and $\overline{a}$ comprise the input to the network. \\
\textbf{Example} Let $f_\emph{up} \equiv \text{running}(c_x)$ and $F_\emph{DL} = \{ f_\emph{up} \}$ for the Sysadmin domain. Then, the abstract state vector $\overline{s_\emph{eg}}$ for the concrete state $s_\emph{eg}$ of the running example would be $\langle 1 \rangle$ and it indicates an abstract state with one computer running. The same vector would be generated for a state where computer $c_1$ was running instead of $c_0$.
Assuming actions are indexed in alphabetical order, nop$()$ would be encoded as $\langle 1, 0 | 0 \rangle$. Similarly, reboot$(c_0)$ would be encoded as $\langle 0, 1 | 1 \rangle$.

Fig.\,\ref{fig:network_architecture} illustrates our process for estimating the Q-values. Given a concrete state $s$
and action $a$, our network (that we call $Q_\emph{GRL}$) predicts the estimated Q-value $Q_\emph{GRL}(\overline{s}, \overline{a}) \approx q_*(s, a)$ by converting $s$ and $a$ to abstract state $\overline{s}$ and action $\overline{a}$ feature vectors based on the feature set $F$. Since the dimensionality of the abstract state and action vectors are fixed and do not depend upon a specific problem instance, the same network can be used to predict Q-values for states and actions across problem instances. This is the key insight into our method for transfer.

The abstract state captures high-level information about the state structure, whereas the abstract action captures the membership of the instantiated objects in the state, allowing our network to learn a generalized, relational Q-function that can be transferred across different problem instances.

\subsection{Generalized Reinforcement Learning}
\label{sec:grl}

Intuitively, Alg.\,\ref{alg:grl} presents our approach for Generalized Reinforcement Learning (GRL). For a given MDP $M$, an initial $Q_\emph{GRL}$ network, and
a set of features $F$, GRL works as follows: Lines $1-5$ transfer knowledge from the  $Q_\emph{GRL}$ network by converting every concrete state $s$ and action $a$ to
abstract state $\overline{s}$ and abstract action $\overline{a}$ vectors using the approach in Sec.\,\ref{subsec:deep_learning_for_value_approximation}. Next, every concrete Q-table entry $Q(s, a)$ is initialized with the predicted value $Q_\emph{GRL}(\overline{s}, \overline{a})$. Note that the Q-table for different problems $M' \not=M$ are different since their state and action spaces are different, however, $\overline{s}$ and $\overline{a}$ are fixed-sized vector representations of any state and action in these problems. This allows $Q_\emph{GRL}$ to transfer knowledge to any problem $M$ with any number of objects.

Lines $9-12$ do Q-learning on the problem $M$ to improve the bootstrapped policy further. Lines $13-16$ further improve the generalization capabilities of  $Q_\emph{GRL}$
by incorporating any policy changes that were observed while doing Q-learning on $M$. GRL returns the task specific policy $Q$ and the updated generalized policy $Q_\emph{GRL}$.

\textbf{Optimization} Lines $2-5$ can be intractable for large problems. We optimized transfer by only initializing entries in a lazy evaluation fashion. An
added benefit is that updates to $Q_\emph{GRL}$ for any abstract state can be easily reflected when encountering a new state that maps to the same abstract state.

\begin{theorem}
Solving problem M using GRL converges under standard conditions of convergence for Q-learning.
\end{theorem}
\begin{proof}[Proof (Sketch)]
The proof is based on the following intuition. $Q_{\emph{GRL}}$ is used to initialize every $Q(s, a)$ entry of $M$ exactly once after which Q-learning operates as normal. The rest of the proof follows from the proof of convergence for Q-learning \cite{DBLP:books/lib/SuttonB98}.
\end{proof}

\begin{algorithm}[t]

\begin{algorithmic}[1]
\REQUIRE MDP $M$, GRL network $Q_{\emph{GRL}}$, features $F$,\\epsilon $\epsilon$, learning rate $\alpha$

\STATE $Q \gets $ initializeEmptyTable$()$
\FOR {$s \in S, a \in A$}
    \STATE $\overline{s}, \overline{a} \gets \text{abstraction}(F, s, a)$
    \STATE $Q(s, a) = Q_\emph{GRL}(\overline{s}, \overline{a})$
\ENDFOR
\STATE $\mathbb{B} \gets $ \text{initialize replay buffer}
\STATE $s \gets s_0$
\WHILE{stopping criteria not met}

    \STATE $a \gets $ getEpsilonGreedyAction$(s)$
    \STATE $s', r \gets $ executeAction$(s, a)$
    \STATE $\delta = r + \gamma \max\limits_{a' \in A} Q(s', a') - Q(s, a) $
    \STATE $Q(s, a) = Q(s, a)  + \alpha \delta$
    \STATE $\overline{s}, \overline{a} \gets $ \text{abstraction}$(F, s, a)$
    \STATE \text{Add} $(\overline{s}, \overline{a}, Q(s, a))$ to $\mathbb{B}$
    
    \STATE Sample a mini-batch $B$ from $\mathbb{B}$
    \STATE Train $Q_{\emph{GRL}}$ using $B$
    \STATE $s' \gets s$
\ENDWHILE

\RETURN $Q, Q_{\emph{GRL}}$
\end{algorithmic}

\caption{Generalized Reinforcement Learning (GRL)}
\label{alg:grl}
\end{algorithm}

\subsection{Scaling Up Q-learning and Transfer}
\label{subsec:leapfrogging}
Transfer capabilities can often be improved if the training strategy uses a curriculum that organizes the tasks presented to the learner
in increasing order of difficulty \cite{DBLP:conf/icml/BengioLCW09}. However, the burden of segregating tasks in order of difficulty often
falls upon a domain expert. 

We adopt \emph{leapfrogging} \cite{DBLP:conf/aips/GroshevGTSA18,DBLP:conf/aaai/KariaS21}, an approach that follows the ``learning-from-small-examples" paradigm by using a problem generator to automatically create a curriculum for learning. Leapfrogging is an iterative process for speeding up learning when used in conjunction with a transfer learning algorithm such as GRL. Leapfrogging is analogous to a loose curriculum, enabling self-supervised training in contrast to curriculum learning that
does not enable automatic self-training.

Leapfrogging operates by initially generating a small problem $M_\emph{small}$ that can be easily solved by vanilla Q-learning without any transfer. It applies GRL (Alg.\,\ref{alg:grl}) to this problem using an uninitialized $Q_{\emph{GRL}}$ network. Once this problem is solved, leapfrogging generates a slightly larger problem and invokes GRL again. The $Q_{\emph{GRL}}$ network from the previous iteration allows GRL to utilize knowledge transfer to solve this new problem relatively quickly while also improving the generalization capabilities of the next generation $Q_{\emph{GRL}}$ network.

\section{Empirical Evaluation}
\label{sec:empirical_evaluation}
We performed an empirical evaluation on four different tasks and our results show that GRL outperforms the baseline in zero-shot transfer performance. We also show that GRL is competitive with approaches receiving additional information like closed-form action models. We now describe the evaluation methodology used for assessing these hypotheses.

\begin{figure*}[ht]
    \centering
    \includegraphics[width=\linewidth]{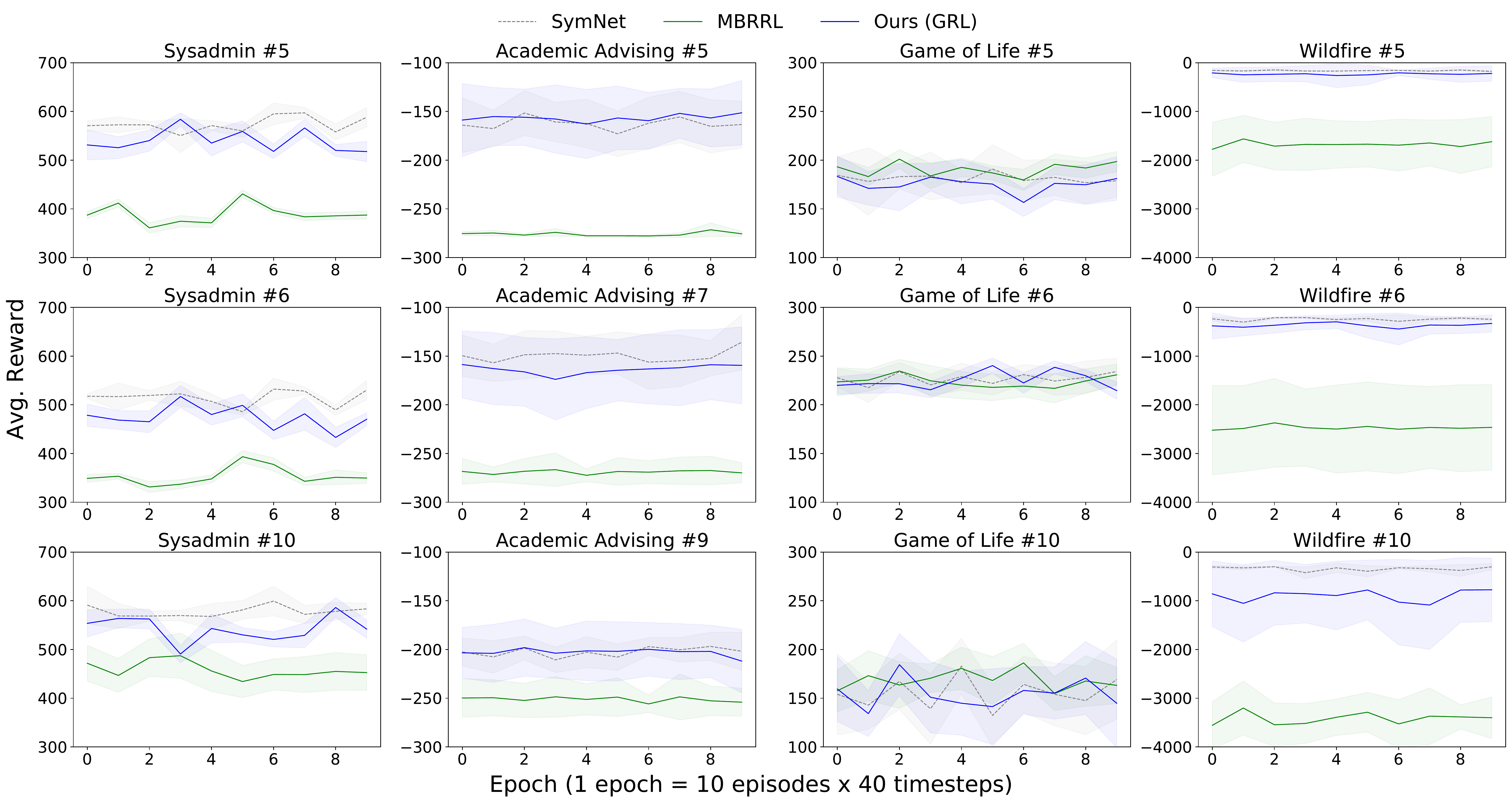}

    \caption{\small Zero-shot transfer performance of GRL (higher values better) compared to MBRRL. The problem number refers to the instance number in the IPPC problem collection. We interpolate the data between any two epochs. We also compare with SymNet, an approach that needs closed-form action models, and thus, is not applicable in our setting. As such, we plot SymNet results as greyed out. We assume minimum reward is obtained when the Q-values are set to NaN as was the case for some problems in Academic Advising for MBRRL. For Academic Advising, we report instances 5, 7 and 9 since instances 6 and 10 contained settings that were incompatible with our system.}
    \label{fig:results}
\end{figure*}

We ran our experiments utilizing a single core and 16 GiB of memory on an Intel Xeon E5-2680 v4 CPU containing 28 cores and 128 GiB of RAM. 

We used the network architecture from Fig.\,\ref{fig:network_architecture} for all of our experiments. Our system is implemented in Python and we used PyTorch \cite{NEURIPS2019_9015} with default implementations of mean squared error (MSE) as the loss function and Adam \cite{DBLP:journals/corr/KingmaB14} as the optimization algorithm for training each domain-specific $Q_{\emph{GRL}}$ network.

Our system uses RDDLsim as the simulator, and thus, accepts problems written in a subset of the Relational Dynamic Influence Definition Language (RDDL) \cite{Sanner:RDDL} .

\subsection{Baselines}
\label{subsec:baselines}

As our baseline, we compare our approach with a first-order Q-function approximation-based approach for transfer; MBRRL \cite{DBLP:conf/aips/NgP21}. We also compare our approach with SymNet \cite{DBLP:conf/icml/GargBM20}, an approach that requires closed-form action models, information that is unavailable to MBRRL and GRL in our setting.\footnote{We thank the authors of SymNet and MBRRL for help in setting up and using their source code.}

MBRRL computes first-order abstractions using conjunctive sets of features and learns a linear first-order approximation of the Q-function over these features. They employ ``mixed approximation" wherein both the concrete $Q(s, a)$ values as well as the approximated Q-values are used to select actions for the policy. 

SymNet uses a Graph Neural Network (GNN) representation of a parsed Dynamic Bayes Network (DBN) for a problem. SymNet thus has access to additional domain knowledge in the form of closed-form action models, and as a result, it is not directly applicable in the RL setting that we consider. Nevertheless, it can serve as a good indicator of the transfer capabilities of GRL that does not need such closed-form representations of action models. We also tried modifying TraPSNet \cite{DBLP:conf/aips/GargBM19}, a precursor of SymNet that does not require action-models, but could not run it due to the limited support for the domains we considered.

\subsection{Tasks, Training and Test Setup}
\label{subsec:tasks}

We consider tasks used in the International Probablistic Planning Competition (IPPC) \cite{Sanner:IPPC}, which have been used by SymNet and MBRRL as benchmarks for evaluating transfer performance. 

Sysadmin: SYS$(n)$ is the Sysadmin domain that was described earlier in the paper. Academic Advising: AA$(l, c, p)$ is a domain where the objective is to pass a certain number of levels $l$ containing $c$ courses each of which need $p$ pre-requisites to be passed first. Game of Life: GoL$(x, y)$ is John Conway's Game of Life environment on a grid of size $x \times y$ \cite{DBLP:journals/scholarpedia/IzhikevichCS15}. Wildfire: WF$(x, y)$ is an environment set in a grid of size $x \times y$. Cells have a chance to burn which also influences their neighbor's probability of burning.





The original versions of Game of Life and Wildfire contain 4-ary predicates that we
converted to an equivalent binary version by converting predicates like \emph{neighbor}$(x_1,y_1,x_1,y_2)$ to \emph{neighbor}$(l_{11},l_{12})$ for use in our evaluation. 

\textbf{Training} For SymNet, we utilized the same problems (IPPC instances 1, 2, and 3) for training as published by the authors. We trained each problem for 1250 episodes. For MBRRL, we utilized the same training procedure as that of the authors wherein we used IPPC instance \#3 for training. We trained each problem for 3750 episodes using Q-learning with an initial $\epsilon = 1$ and a decay rate of $0.997$.

For our leapfrogging approach with GRL, we used a problem generator to generate the problems for training. We used SYS$(3)$, SYS$(4)$, SYS$(6)$; AA$(2, 2, 2)$, AA$(3, 3, 3)$, AA$(4, 4, 4)$; GoL$(2, 2)$, GoL$(3, 3)$, GoL$(4, 4)$, and WF$(2, 2)$, WF$(3, 3)$, WF$(4, 4)$ for training Sysadmin, Academic Advising, Game of Life, and Wildfire respectively. We trained each problem for 1250 episodes using GRL .

\textbf{Testing} We used the same set of problems as SymNet; instances $5-10$ from the IPPC problem collection. The state spaces of these problems are much larger than the training problems used by GRL. For example, the state space size of SYS$(n)$ is $2^n$. For training, the largest problem used by GRL was SYS$(6)$, whereas the test problem, instance 10 (Sysadmin \#10) of the IPPC is SYS$(50)$. Due to space limitations, we report results obtained on instances 5, 6, and 10 and include the others in the supplemental information.

\textbf{Hyperparameters} We used the IPPC horizon $H$ of 40 timesteps after which the simulator was reset to the initial state. 
For GRL, we used Q-learning with $\epsilon = 0.1$.  To train a $Q_{\emph{GRL}}$ network, we used a replay buffer of size $20000$, a mini-batch size of 32, and a training interval of 32 timesteps with 25 steps of optimization per interval. For MBRRL and GRL, we used $\gamma = 0.9$ and $\alpha = 0.05$ for Sysadmin and Game of Life domains as they provide positive rewards. For Academic Advising and Wildfire, we used $\gamma = 1.0$ and $\alpha = 0.3$.

For MBRRL and SymNet, we used the default values of all other settings like network architecture, feature discovery threshold etc. that were published by the authors. 

\textbf{Evaluation Metric} To showcase the efficacy of transfer learning, our evaluation metric compares the performance of MBRRL and our approach after zero-shot transfer.  We freeze the policy after training, transfer it to the test instances and run it greedily for 100 episodes. We report our results using mean and standard deviation metrics computed using 10 individual runs of training and testing.

\textbf{Feature Generation} We used the D2L system \cite{DBLP:conf/aaai/FrancesBG21}, that \emph{requires} a sampled state space as input, for generating description logic features. We modified their code to not require action models and set a complexity bound of $k=5$ for feature generation. We sampled the state space from the first problem used by GRL for training per domain and provided it to D2L for generating the feature set $F_\emph{DL}$.

\subsection{Analysis of Results}
\label{subsec:analysis_of_results}
Our results are shown in Fig.\,\ref{fig:results}. From the results it is easy to see that GRL has excellent zero-shot transfer capabilities and can easily outperform or remain competitive with both MBRRL and SymNet. We now present our analysis followed by a brief discussion of some limitations and future work.

\textbf{Comparison with MBRRL} Our approach is able to outperform MBRRL significantly on Sysadmin, Academic Advising and Wildfire. The DL abstractions used by the GRL
features are more expressive than the conjunctive first-order features used by MBRRL, allowing GRL to learn more expressive policies. Additionally, leapfrogging allows scaling up training and learning of better policies in the same number of episodes in contrast to using a fixed instance for training. 

\textbf{Comparison with SymNet} SymNet utilizes significant domain knowledge. Edges are added between two nodes if an action affects them. This information is unavailable when just observing states as sets of predicates. It is impressive that despite not using such knowledge, GRL is able to remain competitive with SymNet in most of the problems.

\subsection{Limitations and Future Work}
\label{subsec:limitations}

In the Sysadmin domain, the probability with which a computer shuts down depends on how many computers it is connected to. Our representation of $o \in \phi_f(s)$ for representing the action vectors cannot capture such dependencies. However, this is fairly easy to mitigate by a new feature that can count how many computers a specific computer is connected to. We plan to investigate the automatic generation and use of such features in future work.

Leapfrogging requires an input list of object counts for the problem generator that we hand-coded in our experiments. However, we believe that our approach is a step forward in curriculum design by relieving the designer from knowing intrinsic details about the domain, which is often a prerequisite for assessing the difficulty of tasks. The lack of a problem generator can be mitigated by combining leapfrogging with techniques that sample ``subgoals" \cite{DBLP:journals/jair/FernYG06,DBLP:conf/nips/AndrychowiczCRS17} and utilizing GRL to solve the problem and any subsequent problems.

\section{Related Work}
\label{sec:related_work}
Our work adds to the vast body of literature on learning in relational domains. Several of these approaches \cite{DBLP:journals/ai/Khardon99,DBLP:conf/ijcai/GuestrinKGK03,DBLP:conf/aips/WuG07,DBLP:conf/icml/GargBM20} assume that action models are available in an analytical form and thus are not directly applicable to RL settings. For example, FOALP \cite{DBLP:conf/uai/SannerB05} learns features for approximating the value function by regressing over action models. D2L \cite{DBLP:conf/aaai/BonetFG19,DBLP:conf/aaai/FrancesBG21} learns abstract policies assuming an action model where actions can increment or decrement features.  We focus our discussion on relational RL (see \citet{DBLP:conf/nips/TadepalliGD04} for an overview).

\textbf{Q-estimation Approaches} Q-RRL \cite{DBLP:journals/ml/DzeroskiRD01} learns
an approximation of the Q-function by using logical regression trees. GBQL \cite{DBLP:journals/corr/srijata} learns a gradient-boosted tree representation of the Q-function. They evaluated their approach on relatively simple and deterministic tasks, demonstrating the difficulty of transfer using a tree-based approach. Our evaluation consists of harder and stochastic tasks and shows that GRL can learn good policies. 

RePReL \cite{DBLP:conf/aips/KokelMNRT21} uses a high-level planner to learn abstractions for transfer learning. \citet{DBLP:conf/atal/RosenfeldTK17} use hand-crafted features and similarity functions to speed up Q-learning.
MBRRL \cite{DBLP:conf/aips/NgP21} learn conjunctive first-order features for Q-function approximation using hand-coded contextual information for improved performance. 
GRL does not require any hand-coded knowledge.

\textbf{Policy-based approaches} 
\citet{DBLP:journals/jair/FernYG06} use taxonomic syntax with beam search and approximate policy iteration to learn decision-list policies. Their approach uses rollout for estimating Q-values and as such cannot be applied in ``pure" RL settings whereas GRL can. \citet{DBLP:journals/corr/janisch} use graph neural network (GNN) representations of the state to compute policies. GNNs are reliant on the network's receptive field unlike $Q_{\emph{GRL}}$ which uses multilayer perceptrons (MLPs). TraPSNet \cite{DBLP:conf/aips/GargBM19} also uses a GNN and is limited to domains with a single binary predicate and actions with a single parameter. GRL can be used in domains with any number of action parameters and binary predicates.

\textbf{Automatic Curriculum Generation} \citet{DBLP:journals/jair/FernYG06} sample goals from random walks on a single problem. Their approach relies on the target problem to adequately represent the goal distribution for generalization. Similar ideas are explored in \citet{DBLP:conf/aips/FerberHH20} and \citet{DBLP:conf/nips/AndrychowiczCRS17}. These techniques are \emph{intra-instance}, sampling different goals from the same state space and are orthogonal to GRL, that addresses \emph{inter-instance} transfer. 
Our approach of leapfrogging is most similar to that of \citet{DBLP:conf/aips/GroshevGTSA18} and the learn-from-small-examples approach of \citet{DBLP:conf/aips/WuG07}.
We extend their ideas to RL settings and demonstrate its efficacy in transfer. \citet{DBLP:conf/atal/NarvekarS19} automatically generate a curriculum for tasks using hand-crafted features and thus are not applicable in our setting.

\section{Conclusion}
\label{sec:conclusion}
We presented an approach for reinforcement learning in relational domains that can learn good policies with effective zero-shot transfer capabilities.
 Our results show that Description Logic (DL) based features acquired simply through state trajectory sequences can serve to play the role of analytical (closed form) models or make up for some of the performance gap left in their absence. In the future, we plan
to investigate improving the features so that abstract actions can also take into account relationships between the instantiated parameters and the abstract state.

\bibliographystyle{named}
\bibliography{ijcai22}

\clearpage
\appendix

\section{Extended Results}
The complete set of results obtained on instances $5-10$ for each domain of the IPPC can be found in Fig.\,\ref{fig:extended_results}.
For Academic Advising, instances 6, 8, and 10 used a setting that allow 2 concurrent actions to be executed at each timestep.
This setting was incompatible with MBRRL and our software and as such those instances could not be run.

\begin{figure*}[ht]
    \centering
    \includegraphics[width=\linewidth]{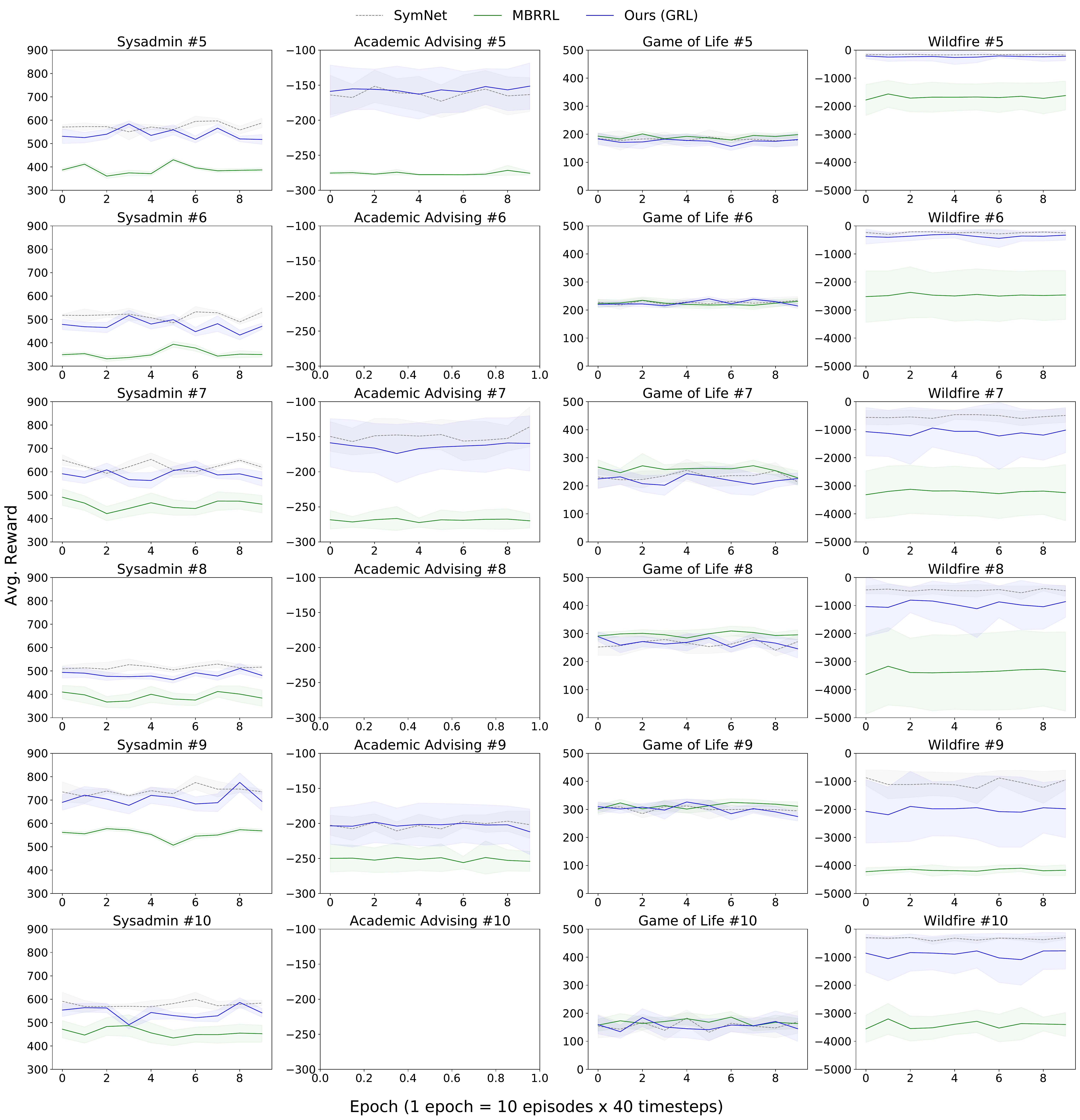}

    \caption{\small Zero-shot transfer results on instances $5-10$ of the IPPC. Even numbered instances of Academic Advising contained settings that were incompatible with MBRRL and GRL and as such could not be run.}
    \label{fig:extended_results}
\end{figure*}

\end{document}